\documentclass{article}

\usepackage{primearxiv}

\usepackage[utf8]{inputenc} 
\usepackage[T1]{fontenc}    
\usepackage{hyperref}       
\usepackage{url}            
\usepackage{booktabs}       
\usepackage{amsfonts}       
\usepackage{nicefrac}       
\usepackage{microtype}      
\usepackage{lipsum}
\usepackage{fancyhdr}       
\usepackage{graphicx}       

\usepackage{xcolor}
\usepackage{todonotes}
\usepackage{amsmath}
\usepackage{amsthm}

\usepackage{subcaption}
\usepackage{caption} 
\usepackage{amsfonts}

\usepackage{listings}
\usepackage{xcolor}       
\usepackage{multirow}
\usepackage{tcolorbox}
\usepackage{tikz}
\usetikzlibrary{decorations.pathreplacing}

\colorlet{lightyellow}{yellow!40}

\makeatletter
{\small 
\xdef\f@size@small{\f@size}
\xdef\f@baselineskip@small{\f@baselineskip}
\normalsize 
\xdef\f@size@normalsize{\f@size}
\xdef\f@baselineskip@normalsize{\f@baselineskip}
}
\newcommand{\smalltonormalsize}{%
  \fontsize
    {\fpeval{(\f@size@small+\f@size@normalsize)/2}}
    {\fpeval{(\f@baselineskip@small+\f@baselineskip@normalsize)/2}}%
  \selectfont
}
\makeatother

\usepackage{amsmath, amssymb, amsthm}
\usepackage{mathtools} 
\newtheorem{theorem}{Theorem}
\theoremstyle{remark}
\newtheorem*{remark}{Remark}
\newtheorem{definition}{Definition}
\newtheorem{corollary}{Corollary}

\DeclareMathOperator*{\argmax}{arg\,max}

\title{The Imitation Game for Educational AI}

\author{
  Shashank Sonkar\thanks{Equal contribution.}, \, Naiming Liu\footnotemark[1], \, Xinghe Chen, \, Richard G. Baraniuk \\
  Rice University \\
  Houston, TX \\
  \texttt{shashank.sonkar@rice.edu} \\
}

\begin{document}
\maketitle

\begin{abstract}
As artificial intelligence systems become increasingly prevalent in education, a fundamental challenge emerges: how can we verify if an AI truly understands how students think and reason? Traditional evaluation methods like measuring learning gains require lengthy studies confounded by numerous variables. We present a novel evaluation framework based on a two-phase Turing-like test. In Phase 1, students provide open-ended responses to questions, revealing natural misconceptions. In Phase 2, both AI and human experts, conditioned on each student's specific mistakes, generate distractors for new related questions. By analyzing whether students select AI-generated distractors at rates similar to human expert-generated ones, we can validate if the AI models student cognition. We prove this evaluation must be conditioned on individual responses - unconditioned approaches merely target common misconceptions. Through rigorous statistical sampling theory, we establish precise requirements for high-confidence validation. Our research positions conditioned distractor generation as a probe into an AI system's fundamental ability to model student thinking - a capability that enables adapting tutoring, feedback, and assessments to each student's specific needs.

\end{abstract}

\section{Introduction}

A fundamental challenge in artificial intelligence for education is developing systems that truly understand how students think and reason \cite{chowdhury2024autotutormeetslargelanguage,corbett1994knowledge,klymkowsky2024endmultiplechoicetests,sleeman}. While recent advances have produced AI systems that can engage with students and curate educational content \cite{sonkar-etal-2023-class,code_class,Markel2023GPTeachIT,schmucker2023ruffle,shridhar2022automaticgenerationsocraticsubquestions,team2024learnlm}, these systems often lack deep understanding of student cognition - how students develop misconceptions, how they reason incorrectly, and why they make specific mistakes \cite{nlet,pedalign,klymkowsky2024endmultiplechoicetests,Shibani_2024}. This understanding is critical for providing meaningful educational support, whether through tutoring, feedback, or assessment. Currently, it's difficult to determine whether an AI system possesses this crucial understanding of student thinking \cite{eedi-mining-misconceptions-in-mathematics,zaphir2024criticallyaithinkframework,malalgoqa}.

\begin{figure}[t!]
    \centering
    \includegraphics[width=\textwidth]{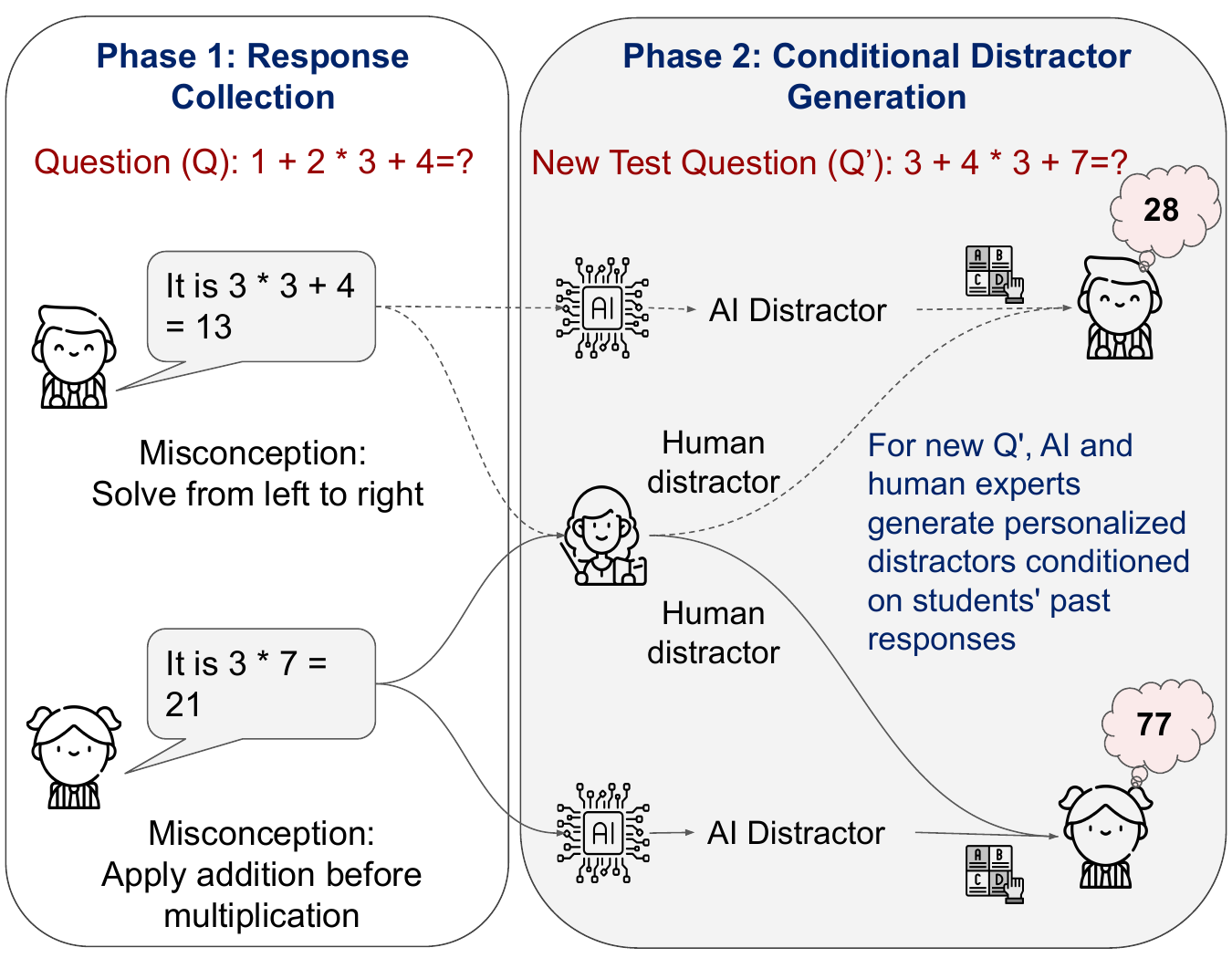}
    \caption{
    A novel Turing-like test for evaluating AI's understanding of student cognition through distractor generation conditioned on individual students' demonstrated misconceptions. In Phase 1 (Response Collection), students solve a mathematical expression Q: 1 + 2 * 3 + 4 = ?, revealing common misconceptions such as solving from left to right (yielding 13) or applying addition before multiplication (yielding 21). In Phase 2 (Conditional Distractor Generation), for a new question Q': 3 + 4 * 3 + 7 = ?, both AI and human experts generate distractors based on observed student misconceptions. The resulting multiple-choice question presents four options: the correct answer, a human expert-generated distractor, an AI-generated distractor, and a random distractor. By analyzing whether students select AI-generated distractors at rates similar to human expert-generated ones, we can rapidly evaluate if the AI system truly understands student cognition. This ability to model and predict student misconceptions is fundamental to providing effective feedback and tutoring, as an AI system that can accurately anticipate how students will misunderstand new concepts can provide targeted, personalized educational support. Note distractor generation must be conditioned on students' prior responses, as unconditioned approaches converge to targeting only the most frequent misconception, failing to test AI's understanding of individual student reasoning. While traditional learning outcome studies require extended periods to validate educational AI systems, this test provides rapid evaluation through conditional distractor generation.}
    \label{fig:Example}
\end{figure}

Traditional evaluation methods like measuring learning gains, while important, suffer from significant limitations. These studies not only require months or even years to conduct \cite{feng2024sampleefficienthumanevaluationlarge,pane2014effectiveness,roschelle2016online}, but their results are confounded by numerous variables - including changes in student motivation and external support, variations in classroom environments and teaching styles, and broader educational policy shifts that occur during the study period \cite{biderman2024lessonstrenchesreproducibleevaluation,laskar2024systematicsurveycriticalreview}. Such confounders make it difficult to isolate whether any observed improvements truly stem from the AI system's deep understanding of student cognition or from other factors entirely.

In contrast to traditional studies with their many confounding variables, we propose a two-phase evaluation methodology to directly test an AI system's understanding of student cognition. In Phase 1, students complete open-ended questions without multiple choice options, providing unbiased samples of natural misconceptions. In Phase 2, for each incorrect answer from Phase 1, the AI must generate a new but related question Q' and predict what wrong answer A' the same student would give to Q'. The key innovation is that these predictions are explicitly conditioned on each student's specific Phase 1 mistakes - without this conditioning, any approach would simply target the most common misconceptions rather than demonstrate understanding of individual reasoning. This creates a personalized test of the AI's ability to model individual student thinking. The student then receives Q' as a multiple choice question containing the correct answer, the AI's predicted wrong answer A', a human expert's predicted wrong answer A'', and a random incorrect answer. By analyzing whether students select AI-generated distractors at rates similar to human expert-generated ones, we can directly validate if the AI system genuinely models student cognition.

This approach mirrors the Turing test in a compelling way: just as the original test evaluates AI through its ability to produce human-like conversational responses, our framework evaluates AI through its ability to predict student mistakes as accurately as expert teachers do. Success requires two fundamental capabilities: understanding common patterns in student misconceptions and predicting how these misconceptions will manifest for individual students across different problems. An AI system that can match human experts in this prediction task demonstrates genuine understanding of student cognition - it's not just recognizing statistical patterns, but modeling the underlying reasoning processes that lead students to specific mistakes.

This understanding of student cognition, as demonstrated through distractor generation conditioned on a student's specific errors, is not just a narrow capability but underpins virtually every aspect of AI-supported education. An AI system that can model how a student reasons about a concept, based on observing their specific mistakes, can better adapt its tutoring, feedback, and assessments to that student's needs. Our research thus positions conditioned distractor generation not merely as an evaluation tool, but as a probe into an AI system's fundamental ability to model student thinking - a capability essential for meaningful educational support. While predicting wrong answers might seem a narrow technical challenge, we show it requires precisely the kind of deep understanding of student cognition that has been missing from current educational AI systems. This paper makes several key contributions:

\begin{enumerate}
    \item \textbf{A Novel Turing-like Test for Educational AI:}
    We introduce a practical and efficient approach to evaluating educational AI systems through a two-phase Turing-like test. Unlike traditional evaluation methods that require lengthy studies of learning outcomes, our test provides a direct measure of an AI system's ability to model student thinking by having it generate new questions and predict errors based on observed student mistakes. The test evaluates whether these AI-generated predictions are indistinguishable from those created by expert human educators, providing a concrete benchmark for measuring progress in educational AI systems' ability to understand individual student reasoning.

    \item \textbf{Positioning Conditioned Distractor Generation as a Core Test of Educational AI Capabilities:}
    Building on this testing methodology, we establish that the ability to generate distractors conditioned on a student's specific errors provides a fundamental test of an educational AI system's capabilities. Through our two-phase methodology - first observing natural student mistakes and then generating new questions with predicted errors - we show that success at this task requires deep modeling of individual student thinking. This provides both a theoretical framework for understanding what capabilities educational AI systems need and a practical method for evaluating these capabilities. Our analysis demonstrates that these same cognitive modeling abilities underpin various educational applications, from providing targeted feedback to designing adaptive assessments, making this test a powerful proxy for evaluating educational AI systems.
\end{enumerate}

\section{Test Design}

Our proposed evaluation framework consists of two distinct phases designed to validate an AI system's understanding of student misconceptions. In the first phase, we collect unbiased samples of student misconceptions through open-ended responses. The second phase tests whether the AI system can accurately predict how these misconceptions will manifest in a newly generated question.

Let $S$ denote our set of students and $D$ our domain of questions. For any student $s \in S$ and question $q \in D$, we define $A(s,q)$ as the student's response and $C(q)$ as the correct answer. 

In Phase 1, each student $s$ responds to a set of questions $Q_s \subset D$ without multiple choice options. For each incorrect response, we record the tuple $(s,q,A(s,q))$ where $A(s,q) \neq C(q)$. These tuples provide unbiased samples of natural student misconceptions, uninfluenced by the presence of pre-selected answer choices.

For Phase 2, given each Phase 1 tuple $(s,q,A(s,q))$:

1. The AI system implements a function $f_{LLM}: D \times A \to D$ that generates a new question $q'$ based on the original question and incorrect answer.

2. Given $q'$, both our AI system and a human expert independently generate predicted wrong answers. The AI system implements $g_{AI}: D \times D \times A \to A$ that generates prediction $a'$, while the human expert implements $g_H: D \times D \times A \to A$ generating prediction $a''$. A random wrong answer $r$ is also generated to serve as a control.

3. The student $s$ then receives question $q'$ as a multiple choice question with four options presented in random order: the correct answer $C(q')$, the AI-predicted wrong answer $a'$, the expert-predicted wrong answer $a''$, and the random wrong answer $r$. The student's selection is denoted $P_2(s,q,q')$.

\section{Statistical Sampling Theory for Personalized Misconception Tracking}
We develop a statistical framework to determine how many students and questions are needed to validate that an AI system can effectively track and predict individual student misconceptions. We show that despite the vast space of possible misconceptions, we can achieve high-confidence validation with surprisingly few students and questions. The key insight is that student misconceptions tend to cluster around common patterns, allowing us to validate an AI system's understanding through careful sampling theory.

\begin{definition}[Misconception Space]
For any topic $T$, let $M(T)$ be the set of all possible misconceptions where $|M(T)| = n$. Each misconception $m_i \in M(T)$ has an associated probability $P(m_i)$ representing its frequency in the student population.
\end{definition}

\begin{theorem}[Misconception Concentration]
For any topic $T$ and error threshold $\epsilon > 0$, there exists a subset $S_k \subset M(T)$ where $|S_k| = k \ll n$ such that:
\[\sum_{m_i \in S_k} P(m_i) \geq (1-\epsilon)\]

This theorem formalizes our observation about clustering: a small set of misconceptions accounts for most student errors.
\end{theorem}

\begin{theorem}[Sample Complexity]
\label{theorem:sample_complexity}
For any topic T, to validate an AI system's understanding of misconceptions in $S_k$ with probability $(1-\delta)$, the required sample size $N$ is:
\[N = O\Bigg(\frac{k \log(1/\delta)}{p_{min}}\Bigg)\]
where $k \ll n$ is the size of $S_k$ and $p_{min}$ is the minimum probability of any misconception in $S_k$.
\end{theorem}

\begin{proof}
The probability of not observing a specific misconception $m_i$ in $N$ trials is $(1-P(m_i))^N \leq (1-p_{min})^N$. By union bound, the probability of missing any misconception in $S_k$ is at most $k(1-p_{min})^N$. Setting this less than $\delta$ and solving:
\[\log(k) + N\log(1-p_{min}) \leq \log(\delta)\]
Using $\log(1-x) \leq -x$, we get $N \geq \frac{\log(k/\delta)}{p_{min}}$.

This bound has profound practical implications: for typical values $(k = 10$ common misconceptions, $\delta = 0.05$ for 95\% confidence level, and $p_{min} = 0.05$ assuming even the rarest common misconception occurs in 1\% of students), this requires approximately 100 students, making testing feasible in practice.

\end{proof}

\begin{theorem}[Question Coverage]
\label{theorem:question_coverage}
For a topic T with k common misconceptions in $S_k$, where each question can test t misconceptions simultaneously, the number of questions Q needed to cover all k misconceptions with probability at least $(1-\delta)$ satisfies:
\[Q = O\left(\frac{k}{t}\ln\left(\frac{k}{\delta}\right)\right)\]
\end{theorem}

\begin{proof}
For any single misconception $m_i$, the probability it is not tested by a given question is $(1-\frac{t}{k})$. After Q questions, the probability of not testing $m_i$ is $(1-\frac{t}{k})^Q$. By union bound, the probability of missing any misconception is at most $k(1-\frac{t}{k})^Q$. Setting this less than $\delta$ and solving:
\[k(1-\frac{t}{k})^Q \leq \delta\]
\[\ln(k) + Q\ln(1-\frac{t}{k}) \leq \ln(\delta)\]
Using $\ln(1-x) \leq -x$:
\[Q \geq \frac{k}{t}\ln(\frac{k}{\delta})\]

\begin{remark}
For typical values ($k = 10$, $t = 2$, $p_{min} = 0.01$, $\delta = 0.05$), this requires approximately $2,500$ total responses, or about $25$ questions per student across $100$ students.
\end{remark}

\end{proof}

\section{Validating AI Prediction of Student Misconceptions}

To meaningfully compare AI and human expert ability to predict student misconceptions, we need precise statistical criteria. Specifically, when a student makes a particular mistake on one question, we want to test whether an AI system can predict their likely mistakes on related questions as accurately as human experts can. This means we need to prove two things: first, that students choose the AI's predicted wrong answers at rates similar to human-expert predicted wrong answers, and second, that both AI and human predictions significantly outperform random guessing. In this section, we establish the formal statistical framework for making these comparisons and determine the sample sizes needed for rigorous validation.

\begin{definition}\label{def:prediction_quality}
For any Phase 1 tuple $(S,Q,A)$ where student $S$ made mistake $A$ on question $Q$, and corresponding Phase 2 question $Q'$, we define:
\[
\mathrm{Choice}(S,Q') = 
\begin{cases}
\mathrm{AI} & \text{if student selects AI-predicted answer } A' \\[0.5em]
\mathrm{Human} & \text{if student selects expert-predicted answer } A'' \\[0.5em]
\mathrm{Random} & \text{if student selects random wrong answer } R \\[0.5em]
\mathrm{Correct} & \text{if student selects correct answer } C
\end{cases}
\]

The prediction quality for each source (AI or Human) is the probability their predicted wrong answer matches the student's actual mistake:
\begin{align}
p_{\mathrm{AI}} &= \mathbb{P}(\mathrm{Choice}(S,Q') = \mathrm{AI} \mid (S,Q,A)) \label{eq:p_ai} \\[0.5em]
p_{\mathrm{Human}} &= \mathbb{P}(\mathrm{Choice}(S,Q') = \mathrm{Human} \mid (S,Q,A)) \label{eq:p_human}
\end{align}
\end{definition}

\begin{theorem}\label{thm:statistical_equivalence}
An AI system's prediction quality matches human expert quality if, with confidence level $1-\alpha$:

\begin{enumerate}
\item The difference in prediction rates is small:
\begin{equation}\label{eq:equiv}
|p_{\mathrm{AI}} - p_{\mathrm{Human}}| \leq \epsilon
\end{equation}

\item Both significantly outperform random guessing:
\begin{align}
p_{\mathrm{AI}} &> p_{\mathrm{Random}} + \delta \label{eq:ai_random} \\
p_{\mathrm{Human}} &> p_{\mathrm{Random}} + \delta \label{eq:human_random}
\end{align}
\end{enumerate}

where:
\begin{itemize}
\item $\epsilon$ is the maximum acceptable difference between AI and human performance
\item $\delta$ is the minimum required improvement over random guessing
\item $p_{\mathrm{Random}} = \frac{1}{4}$ is the probability of random selection from 4 choices
\end{itemize}
\end{theorem}

\begin{proof}
Let $n$ be the total number of Phase 2 responses. For each student-question pair $(S,Q')$, define indicators:

\begin{align}
X_{\mathrm{AI}}(S,Q') &= \begin{cases}
1 & \text{if } \mathrm{Choice}(S,Q') = \mathrm{AI} \\
0 & \text{otherwise}
\end{cases} \label{eq:x_ai} \\[1em]
X_{\mathrm{Human}}(S,Q') &= \begin{cases}
1 & \text{if } \mathrm{Choice}(S,Q') = \mathrm{Human} \\
0 & \text{otherwise}
\end{cases} \label{eq:x_human}
\end{align}

Let $\bar{X}_{\mathrm{AI}}$ and $\bar{X}_{\mathrm{Human}}$ be their respective means across all $n$ responses.

\paragraph{Asymptotic Normality}
By the Lindeberg-Lévy Central Limit Theorem, since $X_{\mathrm{AI}}(S,Q')$ and $X_{\mathrm{Human}}(S,Q')$ are i.i.d. Bernoulli random variables with finite variance:

\begin{align}
\sqrt{n}(\bar{X}_{\mathrm{AI}} - p_{\mathrm{AI}}) &\xrightarrow{d} \mathcal{N}(0, p_{\mathrm{AI}}(1-p_{\mathrm{AI}})) \\
\sqrt{n}(\bar{X}_{\mathrm{Human}} - p_{\mathrm{Human}}) &\xrightarrow{d} \mathcal{N}(0, p_{\mathrm{Human}}(1-p_{\mathrm{Human}}))
\end{align}

For the difference statistic, since responses are paired:
\begin{equation}
\sqrt{n}(\bar{X}_{\mathrm{AI}} - \bar{X}_{\mathrm{Human}} - (p_{\mathrm{AI}} - p_{\mathrm{Human}})) \xrightarrow{d} \mathcal{N}(0, \sigma^2_d)
\end{equation}

where $\sigma^2_d = p_{\mathrm{AI}}(1-p_{\mathrm{AI}}) + p_{\mathrm{Human}}(1-p_{\mathrm{Human}}) - 2\rho\sqrt{p_{\mathrm{AI}}p_{\mathrm{Human}}(1-p_{\mathrm{AI}})(1-p_{\mathrm{Human}})}$ and $\rho$ is the correlation coefficient.

\paragraph{Sample Size Calculation}
For condition \eqref{eq:equiv}, using McNemar's test:
\begin{equation}
Z_1 = \frac{\bar{X}_{\mathrm{AI}} - \bar{X}_{\mathrm{Human}}}{\sqrt{\sigma^2_d/n}} \sim \mathcal{N}(0,1)
\end{equation}

For equivalence testing with margin $\epsilon$:
\begin{equation}
n_1 = \frac{2(z_{\alpha/2})^2\sigma^2_d}{\epsilon^2}
\end{equation}

For conditions \eqref{eq:ai_random} and \eqref{eq:human_random}, using one-sided tests:
\begin{equation}
n_2 = \frac{(z_\alpha)^2p_{\mathrm{Random}}(1-p_{\mathrm{Random}})}{\delta^2}
\end{equation}

The required sample size is therefore:
\begin{equation}
n \geq \max\{n_1, n_2\}
\end{equation}

\paragraph{Relationship to Earlier Sampling Framework}
From Theorem \ref{theorem:sample_complexity}, we need $N = O(\frac{k\log(1/\delta)}{p_{\mathrm{min}}})$ students. From Theorem \ref{theorem:question_coverage}, we need $Q = O(\frac{k}{t}\ln(\frac{k}{\delta}))$ questions per student. The total Phase 2 responses $n$ must satisfy:

\begin{equation}
n \leq N \cdot Q = O\left(\frac{k^2\ln(k)\ln(1/\delta)}{tp_{\mathrm{min}}}\right)
\end{equation}

For typical values ($k=10$, $t=2$, $p_{\mathrm{min}}=0.05$, $\delta=0.05$), this yields $n \approx 400$, which satisfies our calculated sample size requirements when $\epsilon = 0.1$ and $\alpha = 0.05$.
\end{proof}

\begin{remark}
The sample size calculation assumes responses are independent. In practice, there may be correlation between responses from the same student, suggesting use of mixed-effects models for more precise analysis.
\end{remark}

\begin{definition}[Victory Conditions]
Given Phase 2 response data, we say:

\begin{enumerate}
\item AI wins if $p_{\mathrm{AI}} > p_{\mathrm{Human}} + \epsilon$ with confidence $1-\alpha$

\item Human wins if $p_{\mathrm{Human}} > p_{\mathrm{AI}} + \epsilon$ with confidence $1-\alpha$

\item Draw occurs if $|p_{\mathrm{AI}} - p_{\mathrm{Human}}| \leq \epsilon$ with confidence $1-\alpha$
\end{enumerate}

where both $p_{\mathrm{AI}}$ and $p_{\mathrm{Human}}$ must exceed $p_{\mathrm{Random}} + \delta$ for the result to be valid.
\end{definition}

\begin{corollary}[Statistical Test for Victory]
Let $Z_{\mathrm{diff}} = \frac{\bar{X}_{\mathrm{AI}} - \bar{X}_{\mathrm{Human}}}{\sqrt{\sigma^2_d/n}}$

Then:
\begin{itemize}
\item If $Z_{\mathrm{diff}} > z_{1-\alpha}$: AI wins
\item If $Z_{\mathrm{diff}} < -z_{1-\alpha}$: Human wins
\item Otherwise: Draw
\end{itemize}

where $z_{1-\alpha}$ is the $(1-\alpha)$ quantile of the standard normal distribution.
\end{corollary}

\begin{proof}[Proof of Victory Conditions Corollary]
Under the null hypothesis of equal prediction quality ($p_{\mathrm{AI}} = p_{\mathrm{Human}}$), and given the asymptotic normality shown earlier:

\[Z_{\mathrm{diff}} = \frac{\bar{X}_{\mathrm{AI}} - \bar{X}_{\mathrm{Human}}}{\sqrt{\sigma^2_d/n}} \sim \mathcal{N}(0,1)\]

For AI victory, we require:
\begin{align*}
P(p_{\mathrm{AI}} > p_{\mathrm{Human}} + \epsilon) &\geq 1-\alpha \\
\implies P\left(Z_{\mathrm{diff}} > \frac{\epsilon\sqrt{n}}{\sigma_d}\right) &\geq 1-\alpha \\
\implies Z_{\mathrm{diff}} &> z_{1-\alpha}
\end{align*}

Similarly for human victory:
\begin{align*}
P(p_{\mathrm{Human}} > p_{\mathrm{AI}} + \epsilon) &\geq 1-\alpha \\
\implies P\left(Z_{\mathrm{diff}} < -\frac{\epsilon\sqrt{n}}{\sigma_d}\right) &\geq 1-\alpha \\
\implies Z_{\mathrm{diff}} &< -z_{1-\alpha}
\end{align*}

If neither condition holds, we have insufficient evidence to reject the null hypothesis of equal performance, resulting in a draw.

Critically, this test is only valid when both systems outperform random guessing:
\begin{equation}
\min(p_{\mathrm{AI}}, p_{\mathrm{Human}}) > p_{\mathrm{Random}} + \delta
\end{equation}

This ensures we're comparing meaningful prediction strategies rather than random noise.
\end{proof}

\begin{theorem}[Feedback Capability]
An AI system that passes our two-phase test can both identify and generate targeted practice problems for at least a $(1-\epsilon)$ fraction of student misconceptions with probability at least $(1-2\delta)$.
\end{theorem}

\begin{proof}
The proof leverages the two-phase design:
\begin{enumerate}
    \item By Phase 1, we observe student misconceptions through open-ended responses
    \item By Phase 2 success, we prove the AI can:
        \begin{itemize}
            \item Identify the underlying misconception from (S,Q,A)
            \item Generate new question Q' testing the same concept
            \item Predict student response A' to Q'
        \end{itemize}
    \item By the Misconception Concentration theorem, this capability covers $(1-\epsilon)$ of student misconceptions
    \item The ability to generate targeted Q' demonstrates feedback capability through practice problem generation
\end{enumerate}

Therefore, with probability at least $(1-2\delta)$, the system can provide targeted practice through new question generation for at least a $(1-\epsilon)$ fraction of student misconceptions.
\end{proof}

\section{The Insufficiency of Unconditioned Single-Phase Testing}

A natural starting point for evaluating an AI system's understanding of student cognition would be to test its ability to generate plausible wrong answers that students might choose. The intuition is compelling: if an AI can anticipate how students will err, surely it understands how they think. This leads to a simple unconditioned single-phase test where both AI and human experts generate distractors for multiple choice questions without knowledge of individual student reasoning patterns. Let us formalize this design:

\begin{definition}[Unconditioned Single-Phase Test Design]
For any question $q$, let:
\begin{itemize}
   \item $C(q)$ be the correct answer
   \item $A(q)$ be the AI-generated distractor 
   \item $H(q)$ be the human expert-generated distractor
   \item $R(q)$ be a randomly generated incorrect answer
\end{itemize}

The test presents these options in random order to students. Let $S(q,s)$ represent student $s$'s selection for question $q$. Critically, both $A(q)$ and $H(q)$ are generated without observing any prior responses from student $s$.
\end{definition}

\begin{definition}[Performance Metrics]
The quality of AI-generated distractors would be evaluated by comparing selection rates:
\begin{align*}
p_A &= \mathbb{P}(S(q,s) = A(q)) \\
p_H &= \mathbb{P}(S(q,s) = H(q))
\end{align*}
with the AI system "passing" if $|p_A - p_H| \leq \epsilon$ for some small $\epsilon > 0$.
\end{definition}

However, this unconditioned approach has fundamental flaws that make it inadequate for validating true understanding of student cognition:

\begin{theorem}[Unconditioned Convergence]
Given a question $q$ with misconception set $M(q) = \{m_1,...,m_k\}$ where $P(m_i)$ represents the probability of misconception $m_i$ in the student population, both AI and human experts will rationally converge to targeting $m^* = \argmax_{m_i \in M(q)} P(m_i)$.
\end{theorem}

\begin{proof}
For any distractor generator (AI or human), the expected selection rate is maximized by targeting the most common misconception:
\[\mathbb{E}[p] = \sum_{i=1}^k P(m_i)\mathbb{I}(m_i \text{ targeted}) \leq \max_{i} P(m_i)\]
with equality achieved only when targeting $m^*$.
\end{proof}

This convergence creates two critical problems:

\begin{enumerate}
   \item \textbf{Population-Level Optimization}: Both AI and human experts are incentivized to target the single most common misconception for each question, regardless of individual student reasoning patterns. This reduces the test to measuring statistical pattern matching rather than understanding of student cognition.
   
   \item \textbf{False Equivalence}: An AI system could ``pass'' this test by simply learning population-level statistics about common wrong answers, without any actual understanding of how individual students reason about concepts or how their misconceptions evolve across related problems.
\end{enumerate}

These flaws demonstrate why a conditioned two-phase design is necessary - one that can validate an AI system's ability to model individual student reasoning rather than just aggregate statistics. To overcome the limitations of unconditioned testing, our two-phase design introduces crucial conditioning on individual student reasoning. In Phase 1, each student $s$ provides open-ended responses to questions $Q_s$, generating tuples $(s,q,A(s,q))$ where $A(s,q)$ is their incorrect answer. In Phase 2, given each Phase 1 tuple $(s,q,A(s,q))$, both the AI system and human experts observe this specific student mistake and use it to inform their predictions. The AI generates a new question $q'$ and predicts the wrong answer $A'(s,q')$ that this particular student would give, while human experts predict their own expected wrong answer $H'(s,q')$. Both these predictions are explicitly conditioned on the observed student mistake from Phase 1.

This conditioning provides several key advantages:

\begin{theorem}[Prediction Differentiation]
Given students $s_1, s_2$ with different misconceptions $m_1, m_2$ observed in Phase 1, optimal predictions for Phase 2 will differ:
\[A'(s_1,q') \neq A'(s_2,q')\]
even for the same question $q'$.
\end{theorem}

\begin{proof}
Unlike the unconditioned case where predictions converge to the population mode, conditioned predictions should maximize:
\[\mathbb{P}(A'(s,q') \text{ chosen} \mid A(s,q))\]
This probability differs based on the specific misconception demonstrated in Phase 1.
\end{proof}

This conditioning forces the AI to demonstrate three crucial capabilities:
\begin{enumerate}
    \item Understanding individual student reasoning patterns from Phase 1 responses
    \item Generating new questions that test the same conceptual understanding
    \item Predicting how specific misconceptions will manifest in new contexts
\end{enumerate}

Unlike the unconditioned design, success in this framework cannot be achieved through simple population-level statistics, as it requires modeling the cognitive processes of individual students. This deeper understanding of student reasoning pathways has direct implications for educational capabilities: an AI system that can accurately predict how a student's specific misconceptions will manifest across different problems is necessarily equipped to provide targeted tutoring interventions and personalized feedback. The ability to generate new questions that probe and predict individual misconceptions demonstrates the AI system's capacity to construct adaptive learning sequences and deliver feedback that addresses each student's particular cognitive challenges.

\section{Conclusion}

Our research establishes a novel two-phase framework for evaluating educational AI systems through their ability to predict specific student misconceptions. While a simpler unconditioned approach would only validate population-level pattern matching, our two-phase design requires AI systems to demonstrate genuine understanding of individual student cognition. We show that success in this conditioned prediction task implies broader capabilities in providing targeted feedback and personalized tutoring interventions. This work bridges a critical gap in educational AI evaluation by providing immediate, theoretically-grounded validation of an AI system's ability to model student thinking, rather than relying on lengthy outcome studies with numerous confounding variables. By establishing both theoretical guarantees and practical evaluation methods that specifically test understanding of individual student reasoning, our framework lays the foundation for developing AI systems that can truly understand and support each student's unique learning journey - a crucial step toward meaningful AI-supported education.

\section{Limitations and Future Work}

While this paper establishes rigorous mathematical foundations for evaluating educational AI through distractor generation, empirical validation of this framework remains as important future work. The comprehensive theory developed here - from statistical sampling guarantees to proofs of prediction differentiation - provides critical insights into why conditioning on individual student responses is essential and establishes precise validation criteria. This formal foundation addresses fundamental questions about what constitutes genuine understanding of student cognition versus simple pattern matching, creating a principled basis for developing and evaluating AI tutoring systems. Crucially, our theoretical analysis reveals why unconditioned approaches fail to measure true understanding, demonstrating that validating AI systems' educational capabilities requires the type of mathematically grounded, two-phase framework presented here. This theoretical contribution establishes the minimal requirements any evaluation methodology must satisfy to meaningfully assess an AI system's comprehension of student reasoning.

\section*{Acknowledgments}
This work was supported by NSF grant 1842378, ONR grant N0014-20-1-2534, AFOSR grant FA9550-22-1-0060, a Vannevar Bush Faculty Fellowship, OpenAI, and ONR grant N00014-18-1-2047.

\bibliographystyle{unsrt}  
\bibliography{custom,dsp}

\end{document}